\documentclass[10pt]{article} 
\usepackage[preprint]{tmlr}

\usepackage{hyperref}
\usepackage{url}
\usepackage{amsmath, amsfonts, amsthm}
\usepackage[linesnumbered, ruled, vlined]{algorithm2e}
\usepackage{booktabs}
\usepackage{graphicx}
\usepackage{wrapfig}

\usepackage{booktabs, multirow, colortbl, color, array, tabularray, hhline}

\renewcommand{\u}{\ensuremath{\mathbf{u}}}
\renewcommand{\v}{\ensuremath{\mathbf{v}}}
\renewcommand{\a}{\ensuremath{\mathbf{a}}}
\renewcommand{\b}{\ensuremath{\mathbf{b}}}
\renewcommand{\c}{\ensuremath{\mathbf{c}}}
\renewcommand{\d}{\ensuremath{\mathbf{d}}}
\newcommand{\e}{\ensuremath{\mathbf{e}}}
\newcommand{\n}{\ensuremath{\mathbf{n}}}

\newcommand{\m}{\ensuremath{\mathbf{m}}}

\newcommand{\X}{\ensuremath{\mathbf{X}}}
\newcommand{\x}{\ensuremath{\mathbf{x}}}
\newcommand{\y}{\ensuremath{\mathbf{y}}}
\newcommand{\Xtrain}{\ensuremath{\mathbf{X}_\text{train}}}

\newcommand{\ntrees}{\ensuremath{n_\text{trees}}}
\newcommand{\T}{\ensuremath{\mathcal{T}}}
\newcommand{\Trees}{\ensuremath{\{ \T \}_{i=1}^{\ntrees}}}
\newcommand{\Rf}{\ensuremath{\mathcal{R}}}

\renewcommand{\L}{\ensuremath{\mathbb{L}}}
\newcommand{\R}{\ensuremath{\mathbb{R}}}
\newcommand{\B}{\ensuremath{\mathbb{B}}}
\newcommand{\E}{\ensuremath{\mathbb{E}}}
\renewcommand{\P}{\ensuremath{\mathbb{P}}}

\newcommand{\mdot}[1]{\left\langle #1 \right\rangle_\L}
\newcommand{\edot}[1]{\left\langle #1 \right\rangle}

\newcommand{\Ind}[1]{\ensuremath{\mathbf{1}_{#1}}}
\newcommand{\Sign}{\ensuremath{\operatorname{Sign}}}

\newcommand{\hyperdt}{\textsc{HyperDT}}
\newcommand{\hyperrf}{\textsc{HyperRF}}
\newcommand{\sklearn}{\textsc{Scikit-Learn}}

\newcommand{\fhdt}{Fast-\textsc{HyperDT}}

\newtheorem{theorem}{Theorem}[section]
\newtheorem{lemma}[theorem]{Lemma}

\theoremstyle{definition}

\theoremstyle{remark}
\newtheorem{remark}[theorem]{Remark}

\title{
    Even Faster Hyperbolic Random Forests: \\
    A Beltrami-Klein Wrapper Approach
}

\author{\name Philippe Chlenski \email pac@cs.columbia.edu
    \\
    \addr Department of Computer Science\\
    Columbia University
    \AND
    \name Itsik Pe'er \email itsik@cs.columbia.edu\\
    \addr Department of Computer Science\\
    Columbia University
}

\def\month{MM}  
\def\year{YYYY} 
\def\openreview{\url{https://openreview.net/forum?id=XXXX}} 

\begin{document}

\maketitle

\begin{abstract}
    Decision trees and models that use them as primitives are workhorses of machine learning in Euclidean spaces.
    Recent work has further extended these models to the Lorentz model of hyperbolic space by replacing axis-parallel hyperplanes with homogeneous hyperplanes when partitioning the input space.
    In this paper, we show how the \hyperdt\ algorithm can be elegantly reexpressed in the Beltrami-Klein model of hyperbolic spaces.
    This preserves the thresholding operation used in Euclidean decision trees, enabling us to further rewrite \hyperdt as simple pre-- and post-processing steps that form a wrapper around existing tree-based models designed for Euclidean spaces.
    The wrapper approach unlocks many optimizations already available in Euclidean space models, improving flexibility, speed, and accuracy while offering a simpler, more maintainable, and extensible codebase. 
    Our implementation is available at \url{https://github.com/pchlenski/hyperdt}.
\end{abstract}

\section{Introduction}
    All machine learning classifiers implicitly encode geometric assumptions about their feature space.
    Among these, standard decision trees and classifiers built on them stand out as being remarkably indifferent to geometry: Decision tree learning is ultimately a discrete optimization over label statistics, one feature at a time. The metric structure of the multi-feature space only matters when evaluating its final partition.

    In contrast, all
    linear and neural models relying on matrix multiplication implicitly work with signed distances to a decision-boundary hyperplane; probabilistic methods based on Gaussian densities such as Naive Bayes have probability densities that depend on distance to the centroid; not to mention explicitly distance-based methods such as support vector machines
    (SVMs) and $k$-nearest neighbors.
    Yet decision trees rely only on label statistics for candidate partitions which depend only on the ordering of features along each dimension: a property we exploit to reformalize their behavior in hyperbolic space. 

    There are two existing approaches in the literature to extending decision trees to hyperbolic spaces.
    \citet{doorenbos_hyperbolic_2023} proposes \textsc{HoroRF}, an ensemble of horospherical support vector machines \citep{fan_horospherical_2023}, equating between horospheres in hyperbolic and hyperplanes in Euclidean space, as well as between SVMs and linear splits at a decision tree node. 
    In contrast, \citet{chlenski_fast_2024} proposes \hyperdt\, taking a hyperplane-based perspective on Euclidean decision tree learning and following up with a modification of the splitting criterion to accommodate axis-inclined hyperplane splits.
    
    By exploiting the (almost-Euclidean) behavior of hyperplanes in the ambient space of the Lorentz model of hyperbolic space, \hyperdt\ is able to more closely parallel Euclidean tree learning.
    In particular, it is able to consider each spacelike dimension separately, and has identical training and inference complexity to its Euclidean counterparts.
    However, its modified, hyperplane-based splitting criteria are no longer compatible with the threshold-based decision tree algorithms, necessitating its own ad hoc implementation.
    
    In this paper, we rewrite \hyperdt\ as \fhdt, relying on the Beltrami-Klein model of hyperbolic space. 
    This representation allows geodesic decision boundaries to become (axis-parallel) Euclidean hyperplanes, recovering a thresholding-based variant of \hyperdt. \fhdt\ unlocks three advantages:
    \begin{enumerate}
        \item Leveraging optimized implementations is thousands of times faster than existing \hyperdt;
        \item It is intrinsically compatible with well-known Euclidean methods like \sklearn; and
        \item It easily extends to other tree-based paradigms, such as oblique decision trees or \textsc{XGBoost}.
    \end{enumerate}

    \begin{figure}[!t]
        \centering
        \includegraphics[width=\linewidth]{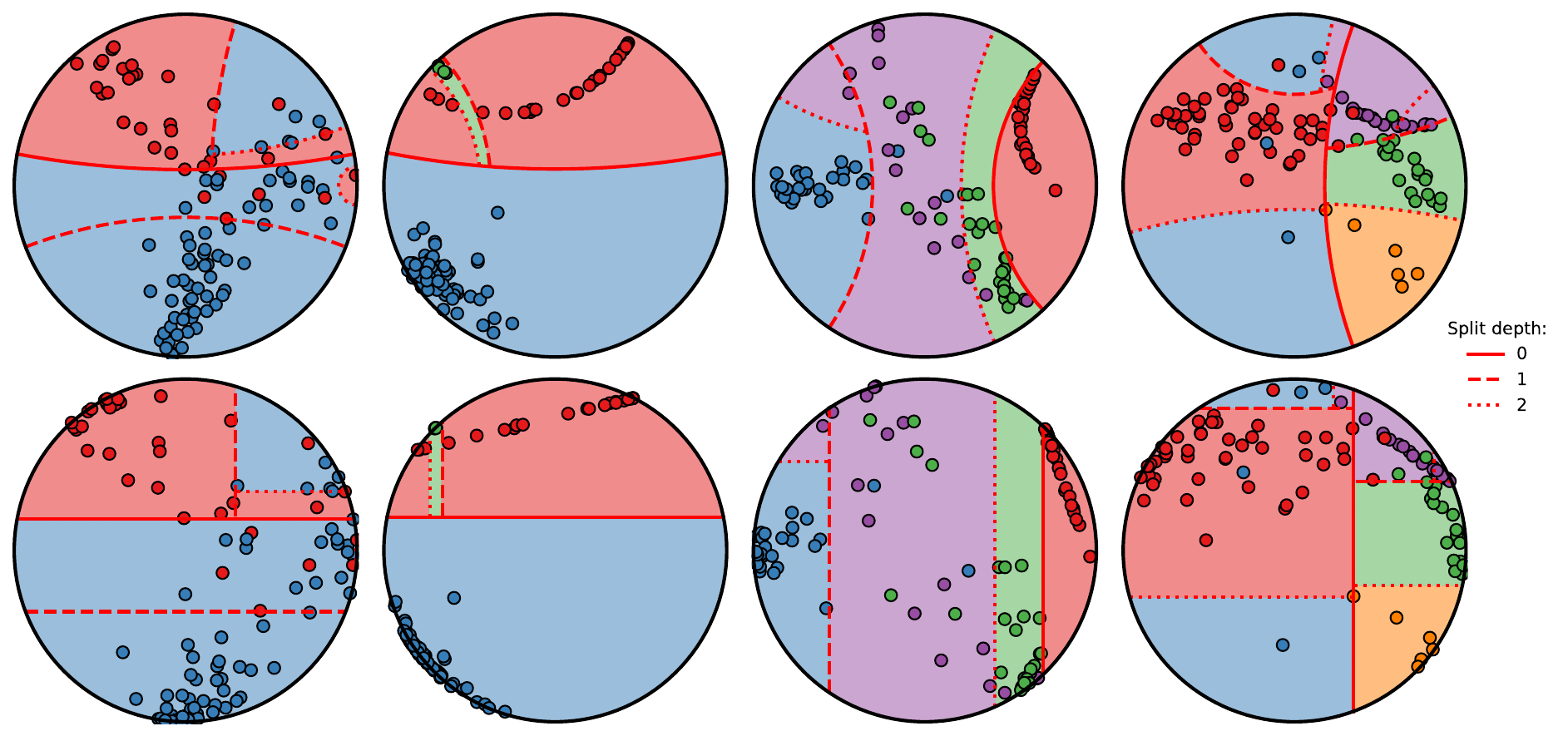}
        \caption{Original \hyperdt\ decision boundaries visualized in the Poincaré model (top) and in the Beltrami-Klein model (bottom). The top half of the figure is a reproduction of Figure 2 in \citet{chlenski_fast_2024}.}
        \label{fig:decision_boundaries}
    \end{figure}

    \subsection{Related work}
        \textbf{Non-Euclidean Decision Trees.}
        This work is most closely related to other works proposing the use of tree-based models in non-Euclidean spaces.
        \citet{chlenski_fast_2024} is the starting point for our model; \citet{chlenski_mixed-curvature_2025} extends decision trees to (products of) any constant-curvature space;
        \citet{doorenbos_hyperbolic_2023} offers a complementary, horosphere-based perspective on random forests in hyperbolic space. 
        \citet{tsagkrasoulis_random_2017} proposes a method to fit random forests for non-Euclidean \textit{labels}, while \citep{bauerschmidt_random_2021} explores connections between random spanning forests and hyperbolic symmetry.

        \textbf{Hyperbolic Classifiers.}
        The problem of classification and regression in hyperbolic spaces has received a lot of attention in the literature.
        Many approaches involve modifying neural networks to work with hyperbolic data \citep{chami_hyperbolic_2019, ganea_hyperbolic_2018, chen_fully_2022, bdeir_hyperbolic_2023, khan_hyperbolic_2024}.
        We are indebted to the literature on hyperbolic SVMs \citep{cho_large-margin_2018, fan_horospherical_2023}; the former makes use of the Klein model as well.
        Regression on hyperbolic manifolds has been studied by \citet{marconi_hyperbolic_2020}.
        
        \textbf{Machine Learning in the Beltrami-Klein Model.} 
        \citet{mao_klein_2024} extends the Klein model for neural nets and provides an excellent overview of the state of the literature, including applications of the Klein model for graph embeddings \citep{mcdonald_heat_2019, yang_towards_2023}, protein sequences \citep{ali_gaussian_2024}, minimal spanning trees \citep{garcia-castellanos_hypersteiner_2025}, and scene images \citep{bi_multiple_2017}.
        We highlight \citet{nielsen_hyperbolic_2010}, an early work using the Klein model to find Voronoi diagrams in hyperbolic space, whose discussion of partitions of the Beltrami-Klein model was foundational for our work.

\section{Preliminaries}
    \subsection{Hyperbolic Space}
        \subsubsection{The Lorentz Model of Hyperbolic Space}
            The Lorentz Model with curvature $K < 0$ in $d$ dimensions, $\L^d_K$,
            is embedded inside Minkowski space, a vector space in $\R^{d+1}$ equipped with the Minkowski inner product:
            \begin{equation}
                \mdot{\u, \v} =
                -u_0v_0 + \sum_{i=1}^d u_iv_i.
            \end{equation}
            Points in the Lorentz Model are constrained to lie on the upper sheet of a two-sheeted hyperboloid with a constant Minkowski inner product $K$, which is also called the curvature:
            \begin{equation}
                \L^d_K = \left\{
                    \u \in \R^{d+1} : \mdot{\u, \u} = \frac{1}{K},\ u_0 > 0
                \right\}.
                \label{eq:hyperboloid_defn}
            \end{equation}
            Distances between two points $\u, \v \in \L$, called geodesic distances, are computed as\footnote{ip: I think we should use arccosh arctan etc. instead of the inverse (negative 1 superscript)}
            \begin{equation}
                \delta_{\L}(\u, \v) = \frac{1}{\sqrt{-K}}\cosh^{-1}\left( 
                    -K\mdot{\u, \v}
                \right).
            \end{equation}
    
        \subsubsection{The Beltrami-Klein Model of Hyperbolic Space}
            The Beltrami-Klein Model  realizes hyperbolic space as the open ball of radius $-1/K$
            \begin{equation}
                \B^d_K = \left\{ 
                    \u \in \R^d : \edot{\u, \u} < -\frac{1}{K}
                \right\}.
            \end{equation}
            
            $\L^d_K$ and $\B^d_K$ are related to through the gnomonic projection $\phi_K: \L^d_K \to \B^d_K$ and its inverse $\phi^{-1}_K: \B^d_K \to \L^d_K$:
            \begin{align}
                \phi_K(u_0, u_1, \ldots, u_d) &= \left(
                    \frac{u_1}{u_0}, \frac{u_2}{u_0}, \ldots, \frac{u_d}{u_0}
                \right)
                \label{eq:phi}
                \\
                \phi^{-1}_K(v_1, \ldots v_d) &= \left( 
                    \frac{1}{\sqrt{1-K\|\v\|^2}},\ \frac{v_1}{\sqrt{1-K\|\v\|^2}},\ \ldots,\ \frac{v_d}{\sqrt{1 - K\|\v\|^2}} 
                \right).
                \label{eq:phi_inverse}
            \end{align}
            Note that the indices for \u\ start at 0, whereas the indices of \v\ start at 1, reflecting the use of an extra dimension in the Hyperboloid model. 
            This correspondence not only maps points but also sends hyperbolic geodesics (great hyperbola arcs on \L) to chords in \B, preserving the hyperbolic distance between points.

            The distance between $\u, \v \in \B$ is computed by composing $\phi^{-1}_K(\cdot)$ and $\delta_\L(\cdot)$ as follows:
            \begin{align}
                \nonumber
                \delta_\B(\u, \v) 
                & = \delta_\L(\phi_K^{-1}(\u), \phi_K^{-1}(\v))\\
                \nonumber
                &= \frac{1}{\sqrt{-K}}\cosh^{-1}\!\left(-K\,\edot{ \phi_K^{-1}(\u), \phi_K^{-1}(\v)}\right)\\
                &= \frac{1}{\sqrt{-K}}\cosh^{-1}\left(-K\frac{
                    1 - \edot{\u, \v}
                }{
                    \sqrt{(1-K\|\u\|^2)(1-K\|\v\|^2)}
                }\right).
            \end{align}

            Equivalently, $\delta_\B$ can be computed using the following geometric construction:
            if the chord through two interior points $\b$ and $\c$ meets the boundary of the Klein disk at points $\a$ and $\b$ (with the ordering $\a, \b, \c, \d$ along the chord), then the hyperbolic distance between $\b$ and $\c$ is given by
            \begin{equation}
                \delta_\B(\b, \c) = \frac{1}{2}\ln\left(\frac{
                    |\a\c| |\b\d|
                }{
                    |\a\b| |\c\d|
                }\right),
            \end{equation}
            where $|\a\b|$ denotes the Euclidean norm of the line segment connecting $\a$ and $\b$, and so on.

            \paragraph{Einstein Midpoints.}
                A key advantage of the Beltrami-Klein model we will make use of in this paper is its closed-form formulation of geodesic midpoints.
                Given two vectors $\u, \v \in \B$, the Einstein midpoint is given by
                \begin{equation}
                    \label{eq:einstein_midpoint}
                    m_\B(\u, \v) 
                    = 
                    \frac{\gamma_{\u}
                    \u + \gamma_{\v}\v}{\gamma_\u + \gamma_\v}, \text{ where }
                    \gamma_\x = \frac{1}{
                        \sqrt{1 - K\|\x\|^2}
                    }.
                \end{equation}
                Note that $\gamma_\x$ is simply the timelike dimension under the inverse gnomonic projection $\phi^{-1}_K$ given in Equation~\ref{eq:phi_inverse}.

                    
    \subsection{HyperDT}
        \hyperdt\ and \hyperrf, which were introduced in \citep{chlenski_fast_2024}, are generalizations of Decision Trees (Classification and Regression Trees, \citep{breiman_classification_2017}) and Random Forests \citep{breiman_random_2001} when $\X \in \L^{n \times d}_K$, $\y \in \R^n$.

        \subsubsection{Geodesically-Convex Splits}
            The central contribution of the \hyperdt\ algorithm is to observe that splits in conventional decision trees can be thought of in terms of dot products with the normal of a separating axis-parallel hyperplane:
            \begin{equation}
                S(\x, i, t) 
                = \Ind{x_d > t}
                = \Sign(\x \cdot \e^{(i)} - t)^+,
            \end{equation}
            where $\e^{(i)}$ is a one-hot vector in dimension $i \in \{0, \ldots d \}$, and $t$ can be thought of as a bias term inducing a translation of the decision boundary by $t$ away from the origin along $\e_i$.
            This observation naturally allows us to consider a more generic split function
            \begin{equation}
                S(\x, \n, t) = \Sign(\x \cdot \n - t),
            \end{equation}
            where $\n \in \R^{d+1}$ is the normal vector for any separating hyperplane, which we will call $H$.
            In \hyperdt, the set of candidate separating hyperplanes is restricted to those whose normal vectors take the form
            \begin{equation}
                \n(i, \theta) = (-\cos(\theta),\ 0,\ \ldots,\ 0,\ \sin(\theta),\ \ldots,\ 0),
            \end{equation}
            where dimensions $0$ and $i$ are the only nonzero dimensions, and where $t=0$, ultimately yielding the \hyperdt\ splitting criterion:
            \begin{equation}
                S(\x, i, \theta) 
                = \Sign(\x \cdot \n(i, \theta))
                = \Sign(\x_i \sin(\theta) - x_0 \cos(\theta)).
            \end{equation}
    
            Removing the bias term ensures geodesic convexity, as only homogeneous hyperplanes (i.e. hyperplanes containing the origin) have geodesically convex intersections with $\L$; on the other hand, parameterizing splits in terms of a spacelike dimension $i$ and an angle $\theta$ allows \hyperdt\ to retain the expressiveness and time complexity of traditional decision tree algorithms. 

        \subsubsection{Hyperbolic Angular Midpoints}
            In CART, if a split falls between $\u, \v \in \Xtrain$, then the boundary is conventionally placed directly in between \u\ and \v.
            That is, for a split in dimension $i$, we set the threshold using a simple arithmetic mean.
            In Euclidean space, this has the property that \u\ and \v\ are now equidistant from the separating hyperplane through $h_i = t$, which is a reasonable inductive bias for choosing where to separate two classes.

            In hyperbolic space, the naive midpoint is likewise the angular bisector
            \begin{equation}
                \theta_\text{naive} = \frac{
                    \theta_1 + \theta_2
                }{
                    2
                },
            \end{equation}
            but in this case the hyperplanes parameterized by $h_0\cos(\theta_1) = h_d\sin(\theta_1)$ and $h_0\cos(\theta_2) = h_d\sin(\theta_2)$ would not intersect $\L$ at points that are equidistant to the hyperplane $h_0\cos(\theta) = h_d\sin(\theta)$ under the hyperbolic distance metric $\delta_\L$.
            In order to ensure equidistance, \hyperdt\ instead uses a more complex midpoint formula derived from $\delta_\L$:
            \begin{equation}
                m_\L(\theta_1, \theta_2) = 
                \cot^{-1}(\alpha - \beta\sqrt{\alpha^2 - 1}), \text{ where }
                \alpha =\frac{
                    \sin(2\theta_1 - 2\theta_2)
                }{
                    2\sin(\theta_1 + \theta_2)\sin(\theta_2 - \theta_1)
                },\ 
                \beta = \Sign\left(
                    \theta_1 + \theta_2 - \pi
                \right).
                \label{eq:hyperbolic_midpoint}
            \end{equation}

\section{Speeding Up HyperDT}
    \begin{algorithm}[!b]
        \caption{Fast-\hyperdt\ training}
        \label{alg:fast_hyperdt}
        \DontPrintSemicolon

\KwIn{Dataset $\X \in \L_K^{n \times d}$, labels $\y \in \R^n$, curvature $K < 0$}
\KwOut{Hyperbolic Random Forest $\Rf = \Trees$ (where $\ntrees=1$ for a single tree)}

\SetKwFunction{FastHyperDT}{Train}
\SetKwFunction{Pre}{Preprocess}
\SetKwFunction{Post}{Postprocess}
\SetKwFunction{AdjustNode}{AdjustThresholds}

\SetKwProg{Fn}{Function}{:}{}
\Fn{\Pre{$\X, K$}}{
    \KwRet $\phi_K(\X)$ \tcp*{Project to Klein model}
}

\Fn{\AdjustNode{$\T, \X, K$}}{
    \If{$\T$ is a leaf}{
        \KwRet
    }
    $d \leftarrow$ $\T$.feature\;
    $t \leftarrow$ $\T$.threshold\;
    $\X \leftarrow \X[\T.\text{subsample\_indices}]$ \tcp*{Random forests may use bootstrapped subsamples of $\X$}
    $\X^+ \leftarrow \{ \x \in \X : x_d > t\}$\;
    $\X^- \leftarrow \{ \x \in \X : x_d \leq t\}$\;
    $L \leftarrow \max_{\x \in \mathbf{X^-}} \{ x_d \}$\;
    $R \leftarrow \min_{\x \in \mathbf{X^+}} \{ x_d \}$\;
    node.threshold $\leftarrow$ \texttt{EinsteinMidpoint}($L, R, K$) \tcp*{Uses Equation~\ref{eq:einstein_midpoint}}
    \AdjustNode{$\T.\mathrm{left}, \X^-$}\;
    \AdjustNode{$\T.\mathrm{right}, \X^+$}\;
}

\Fn{\Post{$\Rf, \X, K$}}{
    \For{$\T \in \Rf$}{
        \AdjustNode{$\T, \X, K, \T$}\;
    }    
    \KwRet $\Rf$\;
}

\Fn{\FastHyperDT{$\X_\L, \y, K$}}{
    $\X_\B \leftarrow$ \Pre{$\X_\L, K$}\;
    $\Rf_\E \leftarrow \texttt{TrainEuclideanModel}(\X_\B, \y)$ \tcp*{For instance, DecisionTreeClassifier.fit()}
    $\Rf_\B \leftarrow$ \Post{$\Rf_\E, \X_\B, K$}\;
    \KwRet $\Rf_\B$\;
}

    \end{algorithm}

    \begin{algorithm}[!t]
        \caption{Fast-\hyperdt\ inference ($O(nh)$ version)}
        \label{alg:fast_hyperdt_inference_Onh}
        \DontPrintSemicolon

\KwIn{Dataset $\X \in \L_K^{n \times d}$, trained model $\Rf_\B$}
\KwOut{Predictions $\hat{\y} \in \R^d$}

\SetKwFunction{FastHyperDT}{Predict}
\SetKwFunction{PredictNode}{PredictNode}

\SetKwProg{Fn}{Function}{:}{}

\Fn{\PredictNode{$\x, \T$}}{
    \If{$\T$ is a leaf}{
        \KwRet $\T.\text{label}$\;
    }
    $d \leftarrow \T.\text{feature}$\;
    $t \leftarrow \T.\text{threshold}$\;
    $\phi(\x)_d \leftarrow x_d / x_0$ \tcp*{Compute Klein model coordinates selectively}
    \eIf{$\phi(\x)_d \leq t$}{
        \KwRet \PredictNode{$\x, \T.\mathrm{left}$}\;
    }{
        \KwRet \PredictNode{$\x, \T.\mathrm{right}$}\;
    }
}

\Fn{\FastHyperDT{$\X_\L, \Rf_\B$}}{
    $\hat{\y} \leftarrow$ array of size $|\X_\L|$\;
    \For{$i \leftarrow 1$ \KwTo $n$}{
        V $\leftarrow \{ \}$ \tcp*{Initialize empty multiset}
        \For{$\T \in |\Rf_\B|$}{
            $v \leftarrow$ \PredictNode{$\X_\B[i], \T$}\;
            V $\leftarrow$ $V \cup \{v\}$ \tcp*{Add prediction to multiset}
        }
        $\hat{\y}[i] \leftarrow \texttt{Resolve}(V)$ \tcp*{However your ensemble aggregates tree votes}
    }
    \KwRet $\hat{\y}$\;
}
    \end{algorithm}
    
    \begin{algorithm}[!t]
        \caption{Fast-\hyperdt\ inference (simple $O(n(d + h))$ version)}
        \label{alg:fast_hyperdt_inference}
        \DontPrintSemicolon

\KwIn{Dataset $\X \in \L_K^{n \times d}$, curvature $K < 0$, trained model $\Rf_\B$}
\KwOut{Predictions $\hat{\y} \in \R^d$}

\SetKwFunction{FastHyperDT}{Predict}
\SetKwFunction{Pre}{Preprocess}

\SetKwProg{Fn}{Function}{:}{}

\Fn{\FastHyperDT{$\X_\L, K, \Rf_\B$}}{
    $\X_\B \leftarrow$ \Pre{$\X_\L, K$}\;
    \KwRet $\texttt{PredictEuclidean}(\X_\B, \Rf_\B)$ \tcp*{For instance, DecisionTreeClassifier.predict()}
}
    \end{algorithm}
    
    The algorithm for speeding up \hyperdt\ training (see Algorithm~\ref{alg:fast_hyperdt}) involves three key steps:
    
    \begin{enumerate}
        \item \textbf{Preprocessing:} To preprocess $\X \in \L^{n \times d}$ for compatibility with Euclidean classifiers, it suffices to project $\X$ to the Beltrami-Klein model $\X_\B \in \B^{n \times d}$ using $\phi_K$ as defined in Equation~\ref{eq:phi}.
        
        \item \textbf{Train using \sklearn:} Train a \sklearn-compatible Decision Tree or Random Forest on the transformed data $\X_\B$, yielding a trained Euclidean predictor $\Rf$.
        
        \item \textbf{Postprocess:} Recompute geodesic midpoints for all decision boundaries in $\Rf$ and store them in the corrected hyperbolic predictor $\Rf_\B$.
        Given two points $\u, \v \in \B,$ the corrected midpoint is simply the Einstein midpoint along the line connecting $\u$ and $\v$, as in Equation~\ref{eq:einstein_midpoint}.
        This modification can be applied recursively to a trained tree, taking $u$ and $v$ as the closest values to either side of the learned threshold.
    \end{enumerate}
    
    This approach allows us to leverage efficient Euclidean implementations of decision trees and random forests (for instance, \sklearn's RandomForestClassifier class) while maintaining the correct geometry of hyperbolic space with curvature $K$.  

    To speed up inference, we propose Algorithm~\ref{alg:fast_hyperdt_inference_Onh}, which preserves the asymptotic complexity of CART and HyperDT (see Theorem~\ref{theorem:complexity}).
    For most practical purposes, it is generally faster to preprocess and use the built-in prediction functionality of the base model instead, as in Algorithm~\ref{alg:fast_hyperdt_inference}. 

    \subsection{Extensions}
        \subsubsection{Poincaré Ball Model}
            It is common for hyperbolic machine learning approaches to use the Poincaré ball model (which we will denote $\P_K$) rather than the hyperboloid or Beltrami-Klein models.
            We omit a detailed discussion of the properties of the Poincaré ball model, except to mention that it is possible to apply our approach to $\X_P \in \P_K^{n, d}$ using a projection $\rho_K: \P_K^d \to \B_K^d$.
            We include its inverse for completeness.
            \begin{align}
                \rho_K(u_1, \ldots, u_d) &= \left(
                    \frac{u_1}{1 + \sqrt{1 - K\|\mathbf{u}\|^2}},\ 
                    \frac{u_2}{1 + \sqrt{1 - K\|\mathbf{u}\|^2}},\ 
                    \ldots,\ 
                    \frac{u_d}{1 + \sqrt{1 - K\|\mathbf{u}\|^2}}
                \right)\\
                \rho_K^{-1}(v_1, \ldots, v_d) &= \left(
                    \frac{2v_1}{1 - K\|\v\|^2},\ 
                    \frac{2v_2}{1 - K\|\v\|^2},\ 
                    \ldots,\ 
                    \frac{2v_d}{1 - K\|\v\|^2}
                \right)
            \end{align}
            Once the points are projected to the Beltrami-Klein model, the rest of the computations proceed exactly as in the hyperboloid case. 

            All HyperDT classifiers are initialized with an \texttt{input\_geometry} hyperparameter, which could be one of \texttt{hyperboloid}, \texttt{klein}, or \texttt{poincare}, and the appropriate conversion is applied during preprocessing.
            
        \subsubsection{XGBoost}
            The extension of our approach to XGBoost \citep{chen_xgboost_2016} follows naturally from the decision tree implementation, as XGBoost fundamentally relies on the same axis-parallel splitting mechanism. 
            While XGBoost introduces gradient boosting, regularization, and sophisticated loss function optimization, the geometric interpretation of its decision boundaries remains unchanged. 
            Each tree in the trained XGBoost ensemble can be postprocessed analogously to a \sklearn\ Random Forest.
            
            The primary implementation challenge lies in accessing individual tree structures within the XGBoost model, as the internal representation differs from \sklearn's. 
            Nevertheless, once the nodes are accessible, the threshold adjustment process remains identical, maintaining the correct hyperbolic geometry throughout the ensemble. 
            We allow XGBoost as a valid backend in our implementation.

            Because XGBoost, unlike \sklearn, does not store a record of sample indices used to train each individual tree, we introduce a \texttt{override\_subsample} hyperparameter to our XGBoost model.
            By default, \texttt{override\_subsample = True}, resetting the \texttt{subsample} hyperparameter of all XGBoost models to 1.0.
            If users choose to set \texttt{override\_subsample = False}, the XGBoost model trains using subsampling, but the postprocessing uses the entire training set.
            In this case, a warning is issued that postprocessing is approximate.

        \subsubsection{LightGBM}
            Similar to XGBoost, LightGBM \citep{ke_lightgbm_2017} is a popular library for gradient-boosted decision trees.
            Also analogous to XGBoost, the innovations of LightGBM do not affect split geometry, so trained LightGBM models can be postprocessed as usual by editing thresholds in LightGBM's proprietary text-based model format.
            The subsampling issue also affects LightGBM models, so if users choose to set \texttt{bagging\_fraction != 0}, the model also issues a warning and performs approximate postprocessing on the entire dataset.
            
        \subsubsection{Oblique Decision Trees}
            For oblique decision trees, which use linear combinations of features rather than axis-parallel splits, our approach requires a natural extension to the projection mechanism. 
            Instead of projecting onto basis dimension axes, we project data points onto the normal vector of each oblique hyperplane.
            Given a hyperplane defined by $\x \cdot \n - t  = 0$, we compute the projection of each point onto the normal vector $\n$. 
            The threshold correction then proceeds identically to the axis-parallel case, applying the Einstein midpoint formula from Equation~\ref{eq:einstein_midpoint} to the scalar projections rather than to individual feature values. 
            
            This generalization maintains the geometric consistency of the decision boundaries in hyperbolic space while leveraging the increased flexibility of oblique splits. 
            The implementation requires modifying the node structure to store the hyperplane parameters $\mathbf{w}$ and applying the projection during both training and inference phases, but the core principle of threshold adjustment remains unchanged.

            Using this projection-based approach, we provide a hyperbolic variant of the HHCart \citep{wickramarachchi_hhcart_2015, wickramarachchi_reflected_2019} and CO2 Forest \citep{norouzi_co2_2015} algorithms, as implemented in the \texttt{scikit-obliquetree} library \citep{ecnu_oblique_2021}.


\section{Theoretical Results}
    \label{sec:theoretical_results}
    In this section, we establish the theoretical foundation for our Fast-\hyperdt\ algorithm by proving its equivalence to the original \hyperdt\ method while demonstrating its computational advantages. 
    The key insight of our approach lies in exploiting the geometric properties of the Beltrami-Klein model, where geodesics appear as straight lines, to leverage efficient Euclidean decision tree implementations.
    

    \begin{lemma}[Hyperplane Classification Equivalence]
        \label{lemma:hyperplane_equivalence}
        Let $\n = (-\cos\theta, 0,...,0,\sin\theta,0,...) \in \mathbb{R}^{d+1}$ be a normal vector for a hyperplane.
        For any $\x \in \mathbb{L}^d_K$, the \hyperdt\ decision rule satisfies:
        \begin{equation}
            \Sign(\edot{\x, \n}) = \Sign\left(\phi_K(\x)_i - \cot(\theta)\right).
        \end{equation}
    \end{lemma}
    \begin{proof}
        \begin{equation}
            \edot{\x, \n} > 0 
            \iff - x_0\cos(\theta) +  x_i\sin(\theta) > 0 
            \iff \frac{x_i}{x_0} > \frac{\cos(\theta)}{\sin(\theta)}
            \iff \phi_K(\x)_i > \cot(\theta).
        \end{equation}
        Since $x_0 > 0$ in the Lorentz model, dividing by $x_0$ never affects the direction of this inequality.
    \end{proof}
    
    \begin{lemma}[Threshold Invariance]
        \label{lemma:threshold_invariance}
        For any decision tree trained to optimize an information gain objective, the feature space partitioning is invariant to the choice of threshold placement between adjacent feature values. Specifically, if $u < v$ are consecutive observed values of a feature, then any threshold $t \in (u, v)$ produces identical tree structures and decision boundaries.
    \end{lemma}
    \begin{proof}
        Let us define the Information Gain (IG) for splitting a dataset $(\X, \y)$ using threshold $t$ on feature $i$:
        \begin{align}
            \nonumber
            \operatorname{IG}(\y, t) &= H(\y) - \left(\frac{|\y^+|}{|\y|}H(\y^+) + \frac{|\y^-|}{|\y|}H(\y^-)\right), \text{ where }\\
            \y^+ &= 
            \left\{y_j \in \y \ :\  (x_{j})_{i} \geq t\right\}, \quad 
            \y^- = 
            \left\{y_j \in \y \ :\  (x_{j})_{i} < t\right\}.
        \end{align}
        Here, $H(\cdot)$ denotes some impurity measure, e.g. Gini impurity.
        Now, for some values $u < v$ such that no $\x \in \X$ has a value between $u$ and $v$ at feature $i$, then the partition of $\y$ into $\y^+$ and $\y^-$ remains identical for any choice of $t \in (u,v)$.
        
        Because information gain depends solely on this partition and not on the specific threshold value, the IG value remains constant for all $t \in (u,v)$. 
        Consequently, any such threshold produces identical tree structures and decision boundaries when optimizing an information gain objective.
    \end{proof}

    \begin{lemma}[Midpoint Equivalence]
        \label{lemma:midpoint_isomorphism}
        Let $\u, \v \in \L^d_{K}$ be two points in the hyperboloid model. 
        Define $\theta_x = \cot^{-1}\left(\frac{x_d}{x_0}\right)$ 
        for any point $\x = (x_0,\ \ldots,\ x_d,\ \ldots) \in \L^d_{K}$. 
        Similarly, let $\phi: \L^d_{K} \to \B^d_{K}$ be the usual gnomonic projection into the Klein model so that $\phi(\x)_d = \frac{x_d}{x_0}$. 
        Then the hyperbolic angular midpoint of $\u,\v$ in $\L^d_{K}$ as defined in Eq.~\ref{eq:hyperbolic_midpoint}
        is the same as the Einstein midpoint of $\phi(\u), \phi(\v)$ in $\B^d_{K}$ as defined in Eq.~\ref{eq:einstein_midpoint}:
        \begin{equation}
        \theta_{\L}(\theta_u, \theta_v) =
        \theta_{m_\B(\phi(\u), \phi(\v))} = 
        \cot^{-1}\left(
            \frac{m_{\B}\left(\phi(\u)_d, \phi(\v)_d\right)_d}{m_{\B}\left(\phi(\u)_d, \phi(\v)_d\right)_0}
        \right)
        \end{equation}
    \end{lemma}
    
    \begin{proof}
        We will show that the hyperbolic angular midpoint in the hyperboloid model and the Einstein midpoint in the Klein model correspond to the same point by verifying that both produce points equidistant from the original endpoints.
        
        In hyperbolic geometry, midpoints are unique: given any two points, there exists exactly one point equidistant from both along the geodesic connecting them. 
        This is implied by the geodesic convexity of hyperbolic spaces \citep{ratcliffe_hyperbolic_2019}. 
        Therefore, if we can show that both constructions yield points in the same 2-D subspace with equal distances to the endpoints, they must be the same point.
        
        Consider two points $\u, \v \in \L^d_K$. 
        Let $\theta_\L(\u, \v)$ be the hyperbolic angular midpoint of $\u$ and $\v$, as defined in Equation~\ref{eq:hyperbolic_midpoint}.
        Let $\m_\L$ denote the unique point in $\L \cap h_0\cos(\theta_L) - h_d\sin(\theta_\L) = 0$.
        Then by the definition of a midpoint,
        \begin{equation}
            \delta_\L(m_\L, \u) = \delta_\L(m_\L, \v).
        \end{equation}
        This property of the hyperbolic midpoint is derived in \citet{chlenski_fast_2024}.
        
        The Einstein midpoint of two points $\mathbf{u}, \mathbf{v} \in \mathbb{L}^n_K$ can equivalently be written as:
        \begin{equation}
            m_\B = \frac{\mathbf{u} + \mathbf{v}}{\|\mathbf{u} + \mathbf{v}\|_{\mathbb{L}}} \cdot \frac{1}{\sqrt{K}},
        \end{equation}
        where $\|\mathbf{w}\|_{\mathbb{L}} = \sqrt{-K\langle \mathbf{w}, \mathbf{w} \rangle_{\mathbb{L}}}$ for a timelike vector $\mathbf{w}$.
        
        Without loss of generality, assume $\mathbf{u} = (u_0,\ 0,\ \ldots,\ u_d,\ \ldots,\ 0)$ and $\mathbf{v} = (v_0,\ 0,\ \ldots,\ v_d,\ \ldots,\ 0)$ where the nonzero components are in positions 0 and $d$.%
        \footnote{
            This is already the case for the axis-parallel splits used in conventional decision trees. 
            For oblique boundaries, it is always possible to rotate the input such that its normal vector is sparse in this way.
        } 
        Simplifying the Einstein midpoint formula,
        \begin{equation}
            m_\B = \frac{(u_0+v_0,\ 0,\ \ldots,\ u_d+v_d,\ \ldots,\ 0)}{K\sqrt{-[(u_0+v_0)^2 - (u_d+v_d)^2]}}
        \end{equation}
        
        To prove equidistance, we compute:
        \begin{align}
            \mdot{m_\B, \u} &= 
            \frac{
                (u_0+v_0)u_0 - (u_d+v_d)u_d
            }{
                K\sqrt{-[(u_0+v_0)^2 - (u_d+v_d)^2]}
            }\\
            \mdot{m_\E, \v} &= \frac{
                (u_0+v_0)v_0 - (u_d+v_d)v_d
            }{
                K\sqrt{-[(u_0+v_0)^2 - (u_d+v_d)^2]}
            }
        \end{align}
        
        For these to be equal, we need:
        \begin{align}
            \nonumber
            (u_0+v_0)u_0 - (u_d+v_d)u_d &= (u_0+v_0)v_0 - (u_d+v_d)v_d\\
            \nonumber
            u_0^2 + u_0v_0 - u_d^2 - u_dv_d &= v_0^2 + u_0v_0 - v_d^2 - u_dv_d\\
            u_0^2 - u_d^2 &= v_0^2 - v_d^2.
        \end{align}
        Since $u_0^2 - u_d^2 = v_0^2 - v_d^2 = 1 / K$ from the hyperboloid constraint (Eq.~\ref{eq:hyperboloid_defn}), the equation is satisfied.
        Therefore, $\delta_{\mathbb{L}}(m_\B, \mathbf{u}) = \delta_{\mathbb{L}}(m_\B, \mathbf{v})$, proving that the Einstein midpoint is equidistant from points $\mathbf{u}$ and $\mathbf{v}$ under the hyperbolic distance.
    \end{proof}

    \begin{remark}
        The three lemmas correspond directly to our algorithm's stages:
        (1) Klein projection enables Euclidean thresholds (Lemma~\ref{lemma:hyperplane_equivalence};
        (2) Invariance permits off-the-shelf training (Lemma~\ref{lemma:threshold_invariance}); and
        (3) Midpoint correction recovers hyperbolically equidistant decision boundaries (Lemma~\ref{lemma:midpoint_isomorphism}).
    \end{remark}
    
    \begin{theorem}[Algorithmic Equivalence]
        \label{theorem:algorithmic_equivalence}
        Assuming that ties in information gain never occur, or are handled identically by \hyperdt\ and \fhdt, both methods produce identical decision boundaries.
        More precisely: letting $\mathcal{D}_H$ be the partition of $\L{}$ induced by \hyperdt{} and $\mathcal{D}_F$ be the partition of $\B{}$ induced by Fast-\hyperdt{} on the same dataset, $\phi(\mathcal{D}_H) \equiv \mathcal{D}_F$.
    \end{theorem}
    \begin{proof}
        We proceed by induction on the depth of a split.

        \textbf{Base case.} At depth 0, the manifold is unpartitioned: $\phi(\mathcal{D}_H) = \phi(\L{}) = \B{} = \mathcal{D}_F$.

        \textbf{Inductive hypothesis.} Assume that we have $\phi(\mathcal{D}_H) = \mathcal{D}_F$ for all $D, F$ of maximum depth $d$.

        \textbf{Inductive step.} 
        By the inductive hypothesis, any region $R \subseteq \L{}$ being split at depth $d+1$ is already equivalent via $\phi(\mathcal{D}_H) = \mathcal{D}_F$.
        By Lemma~\ref{lemma:hyperplane_equivalence}, partitions by equivalent hyperplanes induce equivalent splits, and 
        by Lemma~\ref{lemma:threshold_invariance}, both algorithms learn the same partition of the labels.
        Finally, by Lemma~\ref{lemma:midpoint_isomorphism} the midpoints of the partitioned labels also coincide.
    \end{proof}
    
    \begin{theorem}[Computational Complexity]
        \label{theorem:complexity}
        Let $n$ be the number of training points, $d$ the number of input features, and $h \leq \log n$ the height of each (balanced) tree.  
        With presorted feature lists, \fhdt{} trains in $O(n d \log n)$ time and performs inference in $O(h n)$ time, matching both Euclidean CART and \hyperdt.
        \end{theorem}
        
        \begin{proof}
        Let $t$ denote the number of trees in the ensemble (set $t=1$ for a single tree).
        We analyze one tree and multiply by $t$ at the end.
        
        \textbf{CART.}
        With presorted features, a standard Euclidean decision tree fits in
        $O(d n \log n)$ time and predicts in $O(h n)$ time \citep{sani_computational_2018}.
        The level‑order scan variant costs $O(d n h)$ to fit, but we quote the presorted bound for definiteness.
        Bounding $h \leq \log_2(n)$, a reasonable assumption for balanced trees, equates the two complexities. 
        
        \textbf{\hyperdt.}  
        \hyperdt{} differs from CART by a constant-time modification to the thresholding rule, so its core fitting cost is asymptotically that of CART.
        
        \textbf{Preprocessing.}
        Applying $\phi_K(\cdot)$ to every sample touches each coordinate once, giving $O(n d)$.
        
        \textbf{Postprocessing.}
        Each internal node (at most $2^{h}-1<2n$) needs the two sample values nearest its threshold.
        Because the active sample set halves at every depth, a single linear pass suffices, so this stage is $O(n)$ per depth, giving $O(nh)$ total.
        
        \textbf{Training.} Thus, the total training time for \fhdt\ is
        \begin{equation}
        O(\text{train})
          = O(d n) + O(d n \log n) + O(h n)
          = O(d n \log n)
        \end{equation}
        under the assumption that $h \leq \log_2(n)$.
        Multiplying by $t$ yields $O(t d n \log n)$ for an ensemble.
        
        \textbf{Inference.}
        Algorithm~\ref{alg:fast_hyperdt_inference_Onh} visits exactly $h$ nodes per example and evaluates the projection at a single dimension for each visit, so prediction is $O(h n)$ per tree, or $O(t h n)$ for
        the ensemble.
    \end{proof}

\section{Experimental Results}

    \begin{figure}[!b]
        \centering
        \includegraphics[width=.75\linewidth]{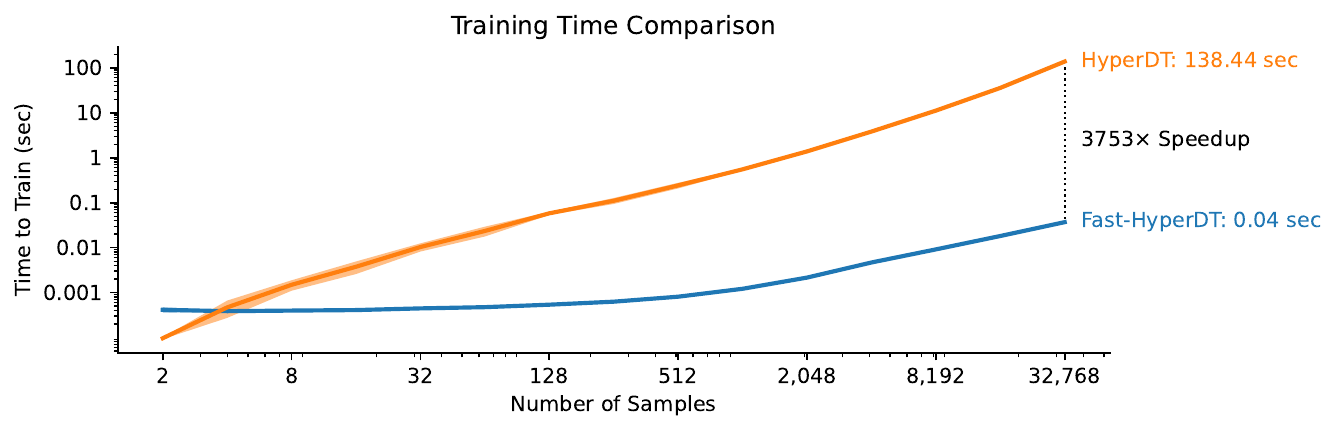}
        \caption{
            We compare the training time of \hyperdt\ and \fhdt\ across different training set sizes. 
            \fhdt\ is consistently faster than \hyperdt\ for 8 or more samples; for 32,768 samples, \fhdt\ trained on average 3,752$\times$ faster.
        }
        \label{fig:speed_comparison}
    \end{figure}
    
    \subsection{Experimental Setup}
    To match \sklearn's tiebreaking behavior more closely, we modified \hyperdt\ by changing the split criterion from $x < t$ to $x \leq t$ and reversing the order in which points are considered.
    We also set the random seed for \sklearn\ decision trees to a constant value.
    This is necessary because, even when \sklearn\ decision trees have subsampling disabled, they still randomly permute the features, which can affect the tiebreaking behavior of the algorithm. 

    Except where otherwise specified, we used trees with a maximum depth of 3 and no further restrictions on what splits are considered (e.g. minimum number of points in a leaf).
    We sample synthetic data from a mixture of wrapped Gaussians \citep{nagano_wrapped_2019}, a common way of benchmarking hyperbolic classifiers \citep{cho_large-margin_2018, chlenski_fast_2024}.
    For regression datasets, we follow \citet{chlenski_manify_2025} in applying cluster-specific slopes and intercepts to the initial vectors sampled from Gaussian distribution to generate regression targets.  

    \subsection{Agreement and Timing Benchmarks}
    
    To evaluate the time to train \fhdt, we trained \hyperdt\ and \fhdt\ decision trees on varying numbers of samples from a mixture of wrapped Gaussian distributions.
    Figure~\ref{fig:speed_comparison} compares the training speeds of \hyperdt\ and \fhdt\ on varying numbers of samples, revealing a 3,752$\times$ speedup when training decision trees on 32,768 samples.

    \begin{figure}[!t]
        \centering
        \includegraphics[width=\linewidth]{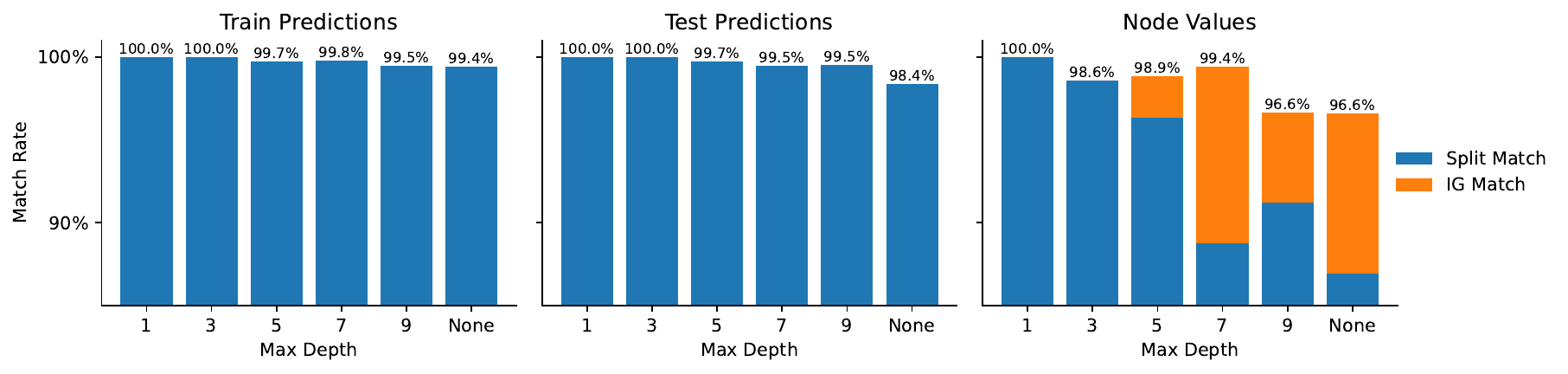}
        \caption{
            We compare the predictions and split values of \hyperdt\ and \fhdt.
            \textbf{Left:} The agreement between both models on the training data.
            \textbf{Middle:} The agreement between both models on the testing data.
            \textbf{Right:} Node-by-node split agreement, distinguishing between exact matches (blue) and matches that are equivalent in information gain (orange). 
        }
        \label{fig:matches}
    \end{figure}

    Given the theoretical results in Section~\ref{sec:theoretical_results}, it is natural to expect the predictions and splits of \hyperdt\ and \fhdt\ to match exactly; instead, Figure~\ref{fig:matches} reveals a high but imperfect degree of correspondence between the two.
    We speculated that differences in pragmatic factors, such as tiebreaking rules (which split to prefer when two splits have the same information gain) and numerical precision, might be behind these slight differences in training behavior. 
    
    To investigate, we tested 10,000 different seeds comparing \hyperdt\ and \fhdt\ on Gaussian mixtures of 1,000 points. 
    Of these, we found that 9,934 splits matched exactly, while 43 splits matched in information gain (up to a tolerance of $0.0001$).
    Figure~\ref{fig:information_gains} shows an example of information gains for 4 sets of features, revealing a tie between the best split along Features 1 and 2.
    Although we attempted to align the tiebreaking behavior of \sklearn\ and \hyperdt\ decision trees, we were never able to perfectly match the splits in all cases.
    
    For the remaining 22 splits, \fhdt\ always achieved a higher information gain than \hyperdt\ did, and angles tended to cluster near $\pi/4$ and $3\pi/4$.
    This seems to suggest that \hyperdt\ suffers from some numerical stability issues for extreme angles, which \fhdt, likely by virtue of omitting the inverse tangent operation used to compute angles in the original \hyperdt\ algorithm, manages to avoid.

    \begin{figure}[!b]
        \centering
        \includegraphics[width=\linewidth]{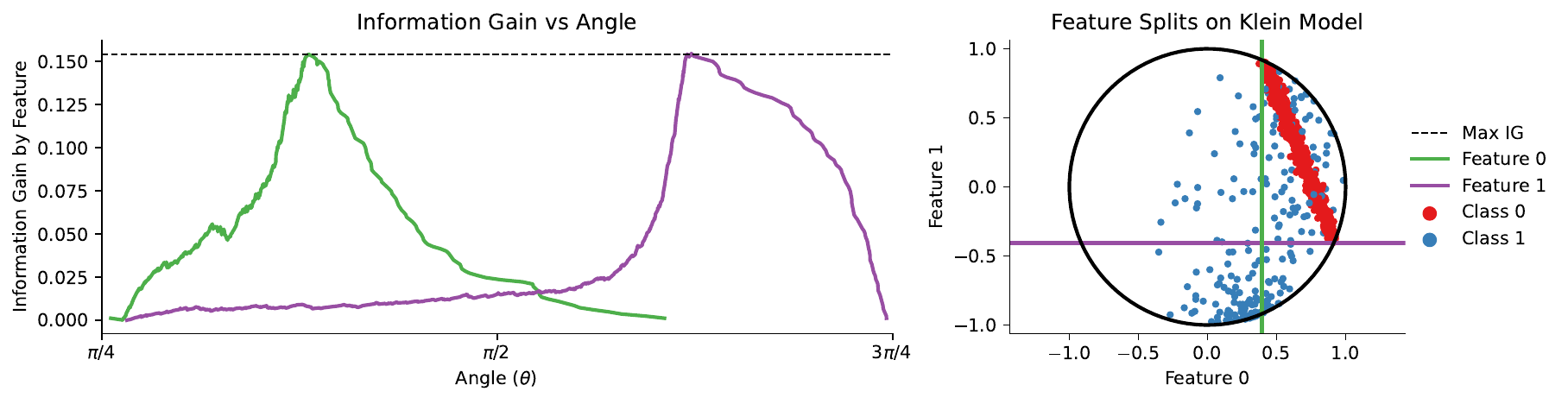}
        \caption{
            A Gaussian mixture dataset for which multiple splits attain the same information gain.
            \textbf{Left:} The information gains attained by each candidate split for each feature.
            \textbf{Right:} Each point in the Klein model, colored by class, and two splits on different features that attain the same information gain.
        }
        \label{fig:information_gains}
    \end{figure}

    \subsection{Other Models}
    In Tables~\ref{tab:benchmarks} and~\ref{tab:regression_benchmarks}, we evaluate \fhdt\ alongside a variety of other models on classification and regression (respectively) of Gaussian mixtures. 
    We evaluate two baselines using \sklearn\ on coordinates in the Lorentz model and the Beltrami-Klein model.
    We evaluate the base model of \fhdt\ against oblique decision trees (CO2 and HHC), as well as XGBoost and LightGBM.

    Oblique decision trees vastly underperformed even the baselines, but are able to learn better-than-random decision rules; for regression, CO2 trees are even competitive with the baseline models.
    In all cases, XGBoost models performed best, highlighting the value of extending this technique to hyperbolic space.

    Finally, we find that base \hyperdt\ performance is only slightly better than \sklearn\ in Beltrami-Klein coordinates.
    Because the latter can be thought of as \fhdt\ without the final postprocessing step, this effectively functions as an ablation of postprocessing.
    This result recapitulates earlier findings in \citep{chlenski_fast_2024} suggesting that ablating the hyperbolic midpoints does not significantly impact \hyperdt\ performance.
    We note that the two ablations actually coincide, as the angular bisector and the average of the Klein coordinates are equal.\footnote{This can be shown by converting angles $\theta_u = \cot^{-1}(u_d/u_0)$ and $\theta_v = \cot^{-1}(v_d/v_0)$ to their bisector $\theta_m = (\theta_u + \theta_v)/2$, then mapping back to Klein coordinates as $\cot(\theta_m) = (u_{d}/x_{0} + x_{d}/x_{0})/2$.}

    \begin{table}[!t]
        \centering
        \caption{
            Classification accuracies across 100 synthetic 8-class benchmarks using Gaussian mixture datasets of varying dimensionalities. 
        }
        \begin{small}
        \begin{tabular}{rcccccccccc}
 & \multicolumn{2}{c}{\sklearn\ DT} & \multicolumn{3}{c}{\fhdt{}} & \multicolumn{2}{c}{\sklearn\ RF} & \multicolumn{3}{c}{Fast-\hyperrf} \\
 \cmidrule(l){2-3} \cmidrule(l){4-6} \cmidrule(l){7-8} \cmidrule(l){9-11}
$d$ & Lorentz & Klein & HyperDT & CO2 & HHC & Lorentz & Klein & HyperRF & LightGBM & XGBoost \\
\midrule
2 & {\cellcolor[HTML]{2F984F}} \color[HTML]{F1F1F1} 48.98 & {\cellcolor[HTML]{238B45}} \color[HTML]{F1F1F1} 50.75 & {\cellcolor[HTML]{238B45}} \color[HTML]{F1F1F1} 50.75 & {\cellcolor[HTML]{F7FCF5}} \color[HTML]{000000} 18.07 & {\cellcolor[HTML]{F7FCF5}} \color[HTML]{000000} 11.44 & {\cellcolor[HTML]{2B934B}} \color[HTML]{F1F1F1} 49.63 & {\cellcolor[HTML]{05712F}} \color[HTML]{F1F1F1} 54.59 & {\cellcolor[HTML]{067230}} \color[HTML]{F1F1F1} 54.34 & {\cellcolor[HTML]{005221}} \color[HTML]{F1F1F1} 57.94 & {\cellcolor[HTML]{00441B}} \color[HTML]{F1F1F1} 59.44 \\
4 & {\cellcolor[HTML]{7AC77B}} \color[HTML]{000000} 41.76 & {\cellcolor[HTML]{6ABF71}} \color[HTML]{000000} 43.24 & {\cellcolor[HTML]{6ABF71}} \color[HTML]{000000} 43.24 & {\cellcolor[HTML]{F7FCF5}} \color[HTML]{000000} 16.68 & {\cellcolor[HTML]{F7FCF5}} \color[HTML]{000000} 13.63 & {\cellcolor[HTML]{45AD5F}} \color[HTML]{F1F1F1} 46.35 & {\cellcolor[HTML]{2B934B}} \color[HTML]{F1F1F1} 49.77 & {\cellcolor[HTML]{2A924A}} \color[HTML]{F1F1F1} 49.99 & {\cellcolor[HTML]{005F26}} \color[HTML]{F1F1F1} 56.66 & {\cellcolor[HTML]{00441B}} \color[HTML]{F1F1F1} 59.60 \\
8 & {\cellcolor[HTML]{A7DBA0}} \color[HTML]{000000} 36.76 & {\cellcolor[HTML]{A0D99B}} \color[HTML]{000000} 37.41 & {\cellcolor[HTML]{A0D99B}} \color[HTML]{000000} 37.41 & {\cellcolor[HTML]{F7FCF5}} \color[HTML]{000000} 15.59 & {\cellcolor[HTML]{F7FCF5}} \color[HTML]{000000} 11.88 & {\cellcolor[HTML]{65BD6F}} \color[HTML]{F1F1F1} 42.76 & {\cellcolor[HTML]{46AE60}} \color[HTML]{F1F1F1} 45.21 & {\cellcolor[HTML]{46AE60}} \color[HTML]{F1F1F1} 45.23 & {\cellcolor[HTML]{006B2B}} \color[HTML]{F1F1F1} 54.10 & {\cellcolor[HTML]{00441B}} \color[HTML]{F1F1F1} 58.06 \\
16 & {\cellcolor[HTML]{C9EAC2}} \color[HTML]{000000} 32.15 & {\cellcolor[HTML]{C8E9C1}} \color[HTML]{000000} 32.29 & {\cellcolor[HTML]{C8E9C1}} \color[HTML]{000000} 32.27 & {\cellcolor[HTML]{F7FCF5}} \color[HTML]{000000} 13.85 & {\cellcolor[HTML]{F7FCF5}} \color[HTML]{000000} 13.46 & {\cellcolor[HTML]{94D390}} \color[HTML]{000000} 37.08 & {\cellcolor[HTML]{72C375}} \color[HTML]{000000} 39.82 & {\cellcolor[HTML]{70C274}} \color[HTML]{000000} 39.83 & {\cellcolor[HTML]{077331}} \color[HTML]{F1F1F1} 49.86 & {\cellcolor[HTML]{00441B}} \color[HTML]{F1F1F1} 54.20 \\
32 & {\cellcolor[HTML]{DEF2D9}} \color[HTML]{000000} 28.66 & {\cellcolor[HTML]{DBF1D6}} \color[HTML]{000000} 28.91 & {\cellcolor[HTML]{DBF1D6}} \color[HTML]{000000} 28.90 & {\cellcolor[HTML]{F7FCF5}} \color[HTML]{000000} 12.38 & {\cellcolor[HTML]{F7FCF5}} \color[HTML]{000000} 12.58 & {\cellcolor[HTML]{AEDEA7}} \color[HTML]{000000} 32.92 & {\cellcolor[HTML]{90D18D}} \color[HTML]{000000} 35.06 & {\cellcolor[HTML]{8BCF89}} \color[HTML]{000000} 35.29 & {\cellcolor[HTML]{117B38}} \color[HTML]{F1F1F1} 44.26 & {\cellcolor[HTML]{00441B}} \color[HTML]{F1F1F1} 48.67 \\
64 & {\cellcolor[HTML]{F1FAEE}} \color[HTML]{000000} 25.73 & {\cellcolor[HTML]{F0F9EC}} \color[HTML]{000000} 25.86 & {\cellcolor[HTML]{EFF9EC}} \color[HTML]{000000} 25.88 & {\cellcolor[HTML]{F7FCF5}} \color[HTML]{000000} 13.10 & {\cellcolor[HTML]{F7FCF5}} \color[HTML]{000000} 12.22 & {\cellcolor[HTML]{C9EAC2}} \color[HTML]{000000} 28.89 & {\cellcolor[HTML]{ABDDA5}} \color[HTML]{000000} 30.45 & {\cellcolor[HTML]{ACDEA6}} \color[HTML]{000000} 30.41 & {\cellcolor[HTML]{0D7836}} \color[HTML]{F1F1F1} 38.23 & {\cellcolor[HTML]{00441B}} \color[HTML]{F1F1F1} 41.02 \\
128 & {\cellcolor[HTML]{F7FCF5}} \color[HTML]{000000} 24.66 & {\cellcolor[HTML]{EFF9EC}} \color[HTML]{000000} 25.56 & {\cellcolor[HTML]{F0F9EC}} \color[HTML]{000000} 25.55 & {\cellcolor[HTML]{F7FCF5}} \color[HTML]{000000} 13.69 & {\cellcolor[HTML]{F7FCF5}} \color[HTML]{000000} 12.87 & {\cellcolor[HTML]{DAF0D4}} \color[HTML]{000000} 26.76 & {\cellcolor[HTML]{B6E2AF}} \color[HTML]{000000} 28.11 & {\cellcolor[HTML]{B8E3B2}} \color[HTML]{000000} 28.07 & {\cellcolor[HTML]{0D7836}} \color[HTML]{F1F1F1} 33.44 & {\cellcolor[HTML]{00441B}} \color[HTML]{F1F1F1} 35.20 \\
\end{tabular}

        \end{small}
        \label{tab:benchmarks}
    \end{table}

    \begin{table}[!t]
        \centering
        \caption{
            Regression mean squared error (MSE) scores across 100 synthetic regression datasets.
            Aside from using regression variants of all models, the benchmarking setup is otherwise identical to Table~\ref{tab:benchmarks}.
        }
        \begin{small}
        \begin{tabular}{rcccccccccc}
 & \multicolumn{2}{c}{\sklearn\ DT} & \multicolumn{3}{c}{\fhdt{}} & \multicolumn{2}{c}{\sklearn\ RF} & \multicolumn{3}{c}{Fast-\hyperrf} \\
 \cmidrule(l){2-3} \cmidrule(l){4-6} \cmidrule(l){7-8} \cmidrule(l){9-11}
$d$ & Lorentz & Klein & HyperDT & CO2 & HHC & Lorentz & Klein & HyperRF & LightGBM & XGBoost \\
\midrule
2 & {\cellcolor[HTML]{F3FAF0}} \color[HTML]{000000} .0297 & {\cellcolor[HTML]{B4E1AD}} \color[HTML]{000000} .0273 & {\cellcolor[HTML]{B5E1AE}} \color[HTML]{000000} .0273 & {\cellcolor[HTML]{ECF8E8}} \color[HTML]{000000} .0293 & {\cellcolor[HTML]{F7FCF5}} \color[HTML]{000000} .1294 & {\cellcolor[HTML]{A0D99B}} \color[HTML]{000000} .0267 & {\cellcolor[HTML]{60BA6C}} \color[HTML]{F1F1F1} .0253 & {\cellcolor[HTML]{5BB86A}} \color[HTML]{F1F1F1} .0252 & {\cellcolor[HTML]{CAEAC3}} \color[HTML]{000000} .0279 & {\cellcolor[HTML]{00441B}} \color[HTML]{F1F1F1} .0214 \\
4 & {\cellcolor[HTML]{F6FCF4}} \color[HTML]{000000} .0299 & {\cellcolor[HTML]{D7EFD1}} \color[HTML]{000000} .0277 & {\cellcolor[HTML]{D8F0D2}} \color[HTML]{000000} .0277 & {\cellcolor[HTML]{F0F9EC}} \color[HTML]{000000} .0293 & {\cellcolor[HTML]{F7FCF5}} \color[HTML]{000000} .1461 & {\cellcolor[HTML]{AFDFA8}} \color[HTML]{000000} .0258 & {\cellcolor[HTML]{72C375}} \color[HTML]{000000} .0236 & {\cellcolor[HTML]{78C679}} \color[HTML]{000000} .0238 & {\cellcolor[HTML]{A3DA9D}} \color[HTML]{000000} .0253 & {\cellcolor[HTML]{00441B}} \color[HTML]{F1F1F1} .0174 \\
8 & {\cellcolor[HTML]{ABDDA5}} \color[HTML]{000000} .0245 & {\cellcolor[HTML]{A5DB9F}} \color[HTML]{000000} .0242 & {\cellcolor[HTML]{A4DA9E}} \color[HTML]{000000} .0242 & {\cellcolor[HTML]{AEDEA7}} \color[HTML]{000000} .0247 & {\cellcolor[HTML]{F7FCF5}} \color[HTML]{000000} .1643 & {\cellcolor[HTML]{53B466}} \color[HTML]{F1F1F1} .0207 & {\cellcolor[HTML]{40AA5D}} \color[HTML]{F1F1F1} .0199 & {\cellcolor[HTML]{40AA5D}} \color[HTML]{F1F1F1} .0200 & {\cellcolor[HTML]{46AE60}} \color[HTML]{F1F1F1} .0202 & {\cellcolor[HTML]{00441B}} \color[HTML]{F1F1F1} .0140 \\
16 & {\cellcolor[HTML]{C1E6BA}} \color[HTML]{000000} .0258 & {\cellcolor[HTML]{B2E0AC}} \color[HTML]{000000} .0251 & {\cellcolor[HTML]{B2E0AC}} \color[HTML]{000000} .0251 & {\cellcolor[HTML]{B2E0AC}} \color[HTML]{000000} .0251 & {\cellcolor[HTML]{F7FCF5}} \color[HTML]{000000} .1821 & {\cellcolor[HTML]{5EB96B}} \color[HTML]{F1F1F1} .0215 & {\cellcolor[HTML]{4EB264}} \color[HTML]{F1F1F1} .0209 & {\cellcolor[HTML]{46AE60}} \color[HTML]{F1F1F1} .0206 & {\cellcolor[HTML]{52B365}} \color[HTML]{F1F1F1} .0210 & {\cellcolor[HTML]{00441B}} \color[HTML]{F1F1F1} .0146 \\
32 & {\cellcolor[HTML]{9ED798}} \color[HTML]{000000} .0241 & {\cellcolor[HTML]{91D28E}} \color[HTML]{000000} .0236 & {\cellcolor[HTML]{91D28E}} \color[HTML]{000000} .0236 & {\cellcolor[HTML]{91D28E}} \color[HTML]{000000} .0235 & {\cellcolor[HTML]{F7FCF5}} \color[HTML]{000000} .2035 & {\cellcolor[HTML]{42AB5D}} \color[HTML]{F1F1F1} .0204 & {\cellcolor[HTML]{38A156}} \color[HTML]{F1F1F1} .0198 & {\cellcolor[HTML]{37A055}} \color[HTML]{F1F1F1} .0198 & {\cellcolor[HTML]{38A156}} \color[HTML]{F1F1F1} .0198 & {\cellcolor[HTML]{00441B}} \color[HTML]{F1F1F1} .0146 \\
64 & {\cellcolor[HTML]{99D595}} \color[HTML]{000000} .0246 & {\cellcolor[HTML]{99D595}} \color[HTML]{000000} .0246 & {\cellcolor[HTML]{99D595}} \color[HTML]{000000} .0246 & {\cellcolor[HTML]{7FC97F}} \color[HTML]{000000} .0236 & {\cellcolor[HTML]{F7FCF5}} \color[HTML]{000000} .2159 & {\cellcolor[HTML]{38A156}} \color[HTML]{F1F1F1} .0210 & {\cellcolor[HTML]{319A50}} \color[HTML]{F1F1F1} .0207 & {\cellcolor[HTML]{329B51}} \color[HTML]{F1F1F1} .0207 & {\cellcolor[HTML]{319A50}} \color[HTML]{F1F1F1} .0206 & {\cellcolor[HTML]{00441B}} \color[HTML]{F1F1F1} .0164 \\
128 & {\cellcolor[HTML]{7AC77B}} \color[HTML]{000000} .0239 & {\cellcolor[HTML]{73C476}} \color[HTML]{000000} .0237 & {\cellcolor[HTML]{72C375}} \color[HTML]{000000} .0237 & {\cellcolor[HTML]{58B668}} \color[HTML]{F1F1F1} .0229 & {\cellcolor[HTML]{F7FCF5}} \color[HTML]{000000} .2261 & {\cellcolor[HTML]{238B45}} \color[HTML]{F1F1F1} .0206 & {\cellcolor[HTML]{18823D}} \color[HTML]{F1F1F1} .0201 & {\cellcolor[HTML]{1A843F}} \color[HTML]{F1F1F1} .0202 & {\cellcolor[HTML]{18823D}} \color[HTML]{F1F1F1} .0201 & {\cellcolor[HTML]{00441B}} \color[HTML]{F1F1F1} .0175 \\
\end{tabular}
        \end{small}
        \label{tab:regression_benchmarks}
    \end{table}

\section{Conclusions}
In this work, we proposed \fhdt, which rewrites \hyperdt\ as a wrapper around Euclidean tree-based models using the Beltrami-Klein model of hyperbolic space.
We prove our method is equivalent to \hyperdt\ while being simpler, faster, and more flexible.
We also demonstrate the superior speed and extensibility of our method empirically, in particular noting that the XGBoost variant of \fhdt\ is vastly more accurate than base \hyperdt.

Future work can focus on extending \fhdt\ to other methods, such as Isolation Forests \citep{liu_isolation_2008} and rotation forests \citep{bagnall_is_2020}, which would address the absence of privileged basis dimensions \citep{elhage_privileged_2023} in hyperbolic embedding methods; extending \fhdt\ to hyperspherical data as in \citet{chlenski_mixed-curvature_2025}; and
extending the connection between \hyperdt, and decision trees to neural networks, as in \citet{aytekin_neural_2022} or via the polytope lens \citep{black_interpreting_2022}.




\bibliography{main}
\bibliographystyle{tmlr}

\appendix

\end{document}